\documentclass[letterpaper]{article}
\usepackage[preprint]{aaai2027}
\usepackage[hyphens]{url}
\usepackage{graphicx}
\usepackage{natbib}
\usepackage{caption}
\usepackage{algorithm}
\usepackage{algorithmic}
\usepackage{booktabs}
\usepackage{amsmath}
\usepackage{amsfonts}
\usepackage{amsthm}
\usepackage{amssymb}
\usepackage{multirow}
\usepackage{tabularx}

\graphicspath{{figures/}}
\newcommand{\method}{\textsc{FedWeave}}
\newcolumntype{Y}{>{\centering\arraybackslash}X}

\newcommand{\maintablesetup}{\small\setlength{\tabcolsep}{0pt}}

\newtheorem{theorem}{Theorem}
\newtheorem{proposition}{Proposition}
\newtheorem{corollary}{Corollary}
\newtheorem{assumption}{Assumption}
\newtheorem{lemma}{Lemma}

\newcommand{\stdnote}[1]{\hspace{0.03em}\raisebox{-0.58ex}{\makebox[0pt][l]{\scriptsize #1}}}
\newcommand{\res}[2]{#1\stdnote{$\pm$#2}}
\newcommand{\bestres}[2]{\textbf{#1}\stdnote{\boldmath$\pm$#2}}
\newcommand{\secondres}[2]{\underline{#1}\stdnote{$\pm$#2}}

\newcommand{\fwbestres}[2]{\bestres{#1}{#2}}
\newcommand{\fwsecondres}[2]{\secondres{#1}{#2}}

\newcommand{\mainmodelcol}{0.75in}

\newcommand{\fwrow}[5]{%
  & \textbf{\method{}}
  & #1
  & #2
  & #3
  & #4
  & #5 \\
}
\newcommand{\panelhead}[2]{\multicolumn{#1}{@{}l@{}}{\normalsize\bfseries #2}}

\title{FedWeave: Rethinking the Unit of Specialization in Heterogeneous Federated MoE-LoRA}

\author{
    Donghang Duan\textsuperscript{\rm 1}\equalcontrib,
    Xu Zheng\textsuperscript{\rm 1}\equalcontrib,
    Lizong Zhang\textsuperscript{\rm 1},
    Chong Mu\textsuperscript{\rm 1},
    Meng Han\textsuperscript{\rm 2}
}
\affiliations{
    \textsuperscript{\rm 1}University of Electronic Science and Technology of China\\
    \textsuperscript{\rm 2}Zhejiang University
}

\begin{document}

\maketitle

\begin{abstract}
Federated PEFT enables LLMs to collaboratively adapt to decentralized private data without sharing raw examples.
However, task heterogeneity across clients can cause cross-task interference and gradient conflicts during aggregation.
Federated MoE-LoRA addresses this challenge through specialized LoRA experts and conditional routing.
Yet existing methods typically specialize at client granularity, implicitly assuming task-coherent clients.
Our core insight is that experts need \emph{purity}---pattern-coherent updates that preserve specialization---whereas routers need \emph{contrast}---mixed-task observations that support expert comparison.
We propose \method{}, a framework that adopts asymmetric aggregation, separating expert aggregation from router optimization to meet these two requirements.
\method{} uses unsupervised prototype discovery to form local buckets and align them across clients, enabling prototype-level expert aggregation while retaining mixed-task client trajectories for router training. 
At inference, \method{} performs sparse inference with one active expert while preserving nearly all soft-routing performance.
Our theoretical analysis explains why asymmetric aggregation is advantageous: it controls expert convergence in stationarity through off-pattern contamination, identifies the consensus error induced by fragmented router trajectories, and bounds sparse-inference risk.
On a heterogeneous multi-task benchmark with mainstream LLM backbones, \method{} consistently outperforms strong baselines, while ablations verify the effectiveness of our design.
\end{abstract}

\section{Introduction}
Federated learning (FL) enables multiple clients to train models collaboratively without sharing their private data~\citep{mcmahan2017fedavg,kairouz2021advances}.
This paradigm is particularly attractive for large language models (LLMs), whose instruction data may be distributed across users and organizations, with direct sharing restricted by policies and regulations on data privacy.
However, full-parameter LLM fine-tuning is often impractical in federated environments because of its substantial computational and communication costs.
Parameter-efficient fine-tuning (PEFT), especially low-rank adaptation (LoRA)~\citep{hu2022lora}, addresses this constraint by updating lightweight adapters while keeping the backbone frozen.
Recent work has therefore explored federated instruction tuning and federated LoRA as scalable strategies for decentralized LLM adaptation~\citep{zhang2023fedit,sun2024ffa,guo2025selective,bian2026fedalt,bian2026fedtreelora}.

Despite this progress, federated LLM instruction tuning remains fundamentally challenged by task heterogeneity.
Clients may differ in their domains, task distributions, and instruction styles; more importantly, a single client may itself mix several latent instruction patterns~\citep{marfoq2021mixture,feng2023fedins,talasso2026fedrouter}.
These patterns may demand different reasoning skills, output structures, or adaptation directions, despite residing on the same client.
When these patterns are compressed into a single client update, their adaptation signals lose the pattern coherence required for specialization before cross-client aggregation begins.

Nevertheless, current efforts on federated PEFT mainly organize heterogeneity at the granularity of clients.
Client correction, clustering, and personalization improve how client updates are optimized or shared~\citep{li2020fedprox,sattler2021cfl,sun2024ffa,guo2025selective}.
Mixture-of-Experts (MoE) and multi-expert LoRA add specialization capacity by routing inputs to lightweight experts~\citep{shazeer2017moe,fedus2022switch,liao2025hmora}.
Federated expert models bring this idea to decentralized settings~\citep{zec2021specialized,reisser2021fedmix,wang2025fedlease}.
Yet when expert allocation, training, and aggregation remain tied to client identity, incompatible patterns within one client are mixed into the same update and related patterns across clients remain unaligned.

The central challenge is a component-dependent mismatch in aggregation granularity. 
Put simply, experts need \emph{purity}---pattern-coherent updates that preserve specialization---whereas routers need \emph{contrast}---cross-pattern observations that enable expert comparison. 
Consequently, client-level aggregation is too coarse for experts, while independently training and aggregating a router for each pattern fragments the cross-pattern trajectory required for routing. No single aggregation granularity serves both components. 
Moreover, realizing such mixed-granularity aggregation is non-trivial: training budgets are limited and unevenly distributed across patterns, while experts and routers exhibit distinct convergence dynamics.

\method{} operationalizes this purity--contrast insight through asymmetric aggregation as illustrated in Figure~\ref{fig:fedweave_framework}.
Each client performs unsupervised prototype discovery without requiring task-identity labels and derives an adaptation signature for each local bucket.
The server uses these signatures to align related buckets across clients into global expert groups.
Local training interleaves bucket-homogeneous mini-batches, after which the server aggregates expert updates among matched prototypes.
In contrast, the broadcast global router is updated across each client's complete bucket trajectory and aggregated from one client-level delta, preserving cross-pattern contrast for expert selection.
\method{} trains with soft routing and supports both soft-mixture and sparse inference; in the sparse mode, only one LoRA expert is activated.

Our component-wise theory makes this asymmetry explicit: prototype-level aggregation gives $O(1/U)$ expert convergence to a contamination-controlled stationary neighborhood under non-vanishing routed weight; the router result gives a PL-convergence bound separating interleaving-order bias from reset consensus error; gap calibration then transfers these bounds to sparse-inference risk.
On a multi-task benchmark comprising CoEdIT, GSM8K, TweetEval, and ARC-C, \method{} achieves the best performance among the evaluated baselines on Llama3.2-3B and Gemma-2-2B and outperforms the closest client-level expert baseline across the tested heterogeneity levels~\cite{meta2024llama3p2,gemma2024gemma2}, while ablations verify the effectiveness of our design.
Sparse inference preserves nearly all soft-routing quality while activating only one expert per input.

Our contributions are summarized as follows:
\begin{itemize}
    \item We identify a purity--contrast principle for federated MoE-LoRA under intra-client task heterogeneity: experts require pattern-coherent updates for specialization, whereas routers require client-level trajectories that retain cross-pattern contrast.

    \item We propose \method{}, a federated MoE-LoRA framework which refines the granularity to the prototype level and aggregates expert and router in an asymmetric way.

    \item We derive component-wise convergence guarantees: contamination-controlled expert stationarity, PL router convergence separating order bias from reset consensus error, and sparse-risk transfer.
\end{itemize}

\section{Related Work}

\paragraph{Federated learning under heterogeneity.}
FedProx and SCAFFOLD correct client-level optimization~\citep{li2020fedprox,karimireddy2020scaffold}, while clustered FL learns models for latent client groups~\citep{sattler2021cfl,ghosh2020ifca}; both retain the client-level update structure of FedAvg~\citep{mcmahan2017fedavg,kairouz2021advances}.
Mixture-based FL instead models each client through latent sources~\citep{marfoq2021mixture}, and FedIns performs instance-adaptive inference under intra-client heterogeneity~\citep{feng2023fedins}.
\method{} instead turns aligned within-client prototypes into shared expert-update units.

\paragraph{Federated PEFT for LLMs.}
LoRA and QLoRA restrict adaptation to low-rank residuals~\citep{hu2022lora,dettmers2023qlora}, enabling federated instruction tuning~\citep{zhang2023fedit,yang2025fedlorasurvey}.
Federated LoRA methods address factor-aggregation bias through fixed or selective factors, server correction, alternating updates, exact inner-matrix aggregation, or gauge-aware consensus subspaces~\citep{sun2024ffa,guo2025selective,bian2025lorafair,koo2025loraa2,singhal2025fedsb,chen2026glora}.
Others accommodate heterogeneous ranks or structures, privacy noise, personalization, or layer-wise sharing~\citep{cho2024hetlora,wang2024flora,fan2025helora,lee2025fedsvd,bian2026fedalt,hao2025pf2lora,shen2026sdflora}.
FedTreeLoRA varies layer-wise sharing depth through a client hierarchy to couple statistical and functional heterogeneity~\citep{bian2026fedtreelora}.
These methods modify adapter parameterization or client-level sharing, whereas \method{} changes the within-client data unit that defines expert specialization.

\paragraph{Expert models and sparse routing.}
Sparse MoE models~\citep{shazeer2017moe,fedus2022switch} and MoE-LoRA~\citep{liao2025hmora} provide conditional specialization; LLaVA-MoLE and HotMoE address conflicts through sparse or hybrid routing~\citep{chen2024llavamole,huang2026hotmoe}, while sparse-and-orthogonal federated LoRA targets multi-task interference through implicit expert separation~\citep{yang2026sparseorthogonal}.
Federated MoE methods personalize experts by client~\citep{zec2021specialized,reisser2021fedmix}: FedMoE selects client-specific sub-MoEs from a pretrained sparse model~\citep{mei2024fedmoe}, while FedLEASE allocates LoRA experts to client clusters~\citep{wang2025fedlease}.
LoMo-Fed separates global and local experts in personalized vision heads~\citep{sang2026lomofed}, whereas FedRouter associates local clusters with task-centric adapters but uses nearest-centroid selection at evaluation~\citep{talasso2026fedrouter}.
In contrast, \method{} aggregates expert updates over aligned within-client prototypes while training a single global router from complete client trajectories, matching expert purity with router contrast.

\begin{figure*}[t]
    \centering
    \includegraphics[width=0.90\textwidth]{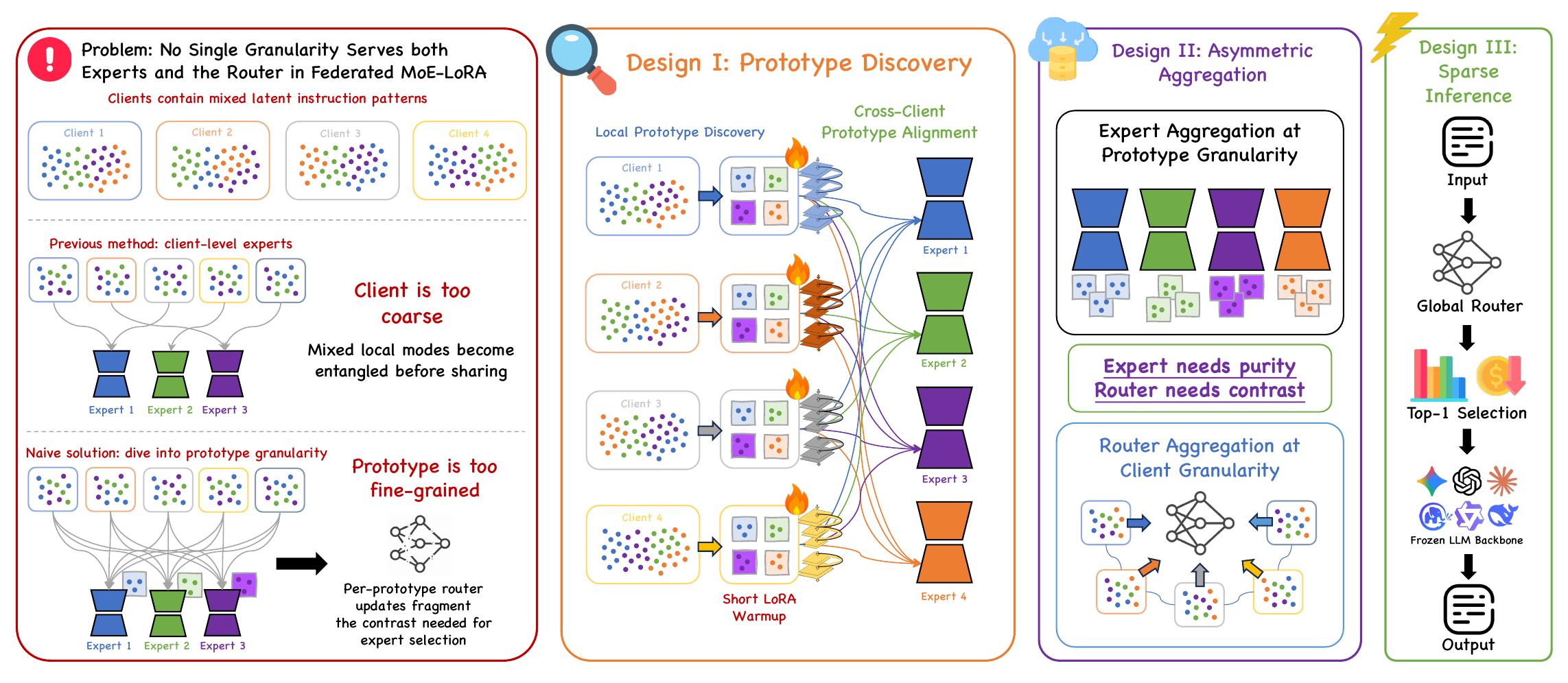}
    \caption{\textbf{FedWeave overview.} Aligned local prototypes provide pattern-coherent batch supervision for experts, whereas the single global router retains cross-pattern contrast by training over each client's complete bucket trajectory. The server aggregates the two components at different granularities.}
    \label{fig:fedweave_framework}
\end{figure*}

\section{Method}
\subsection{Problem Setup}
We consider $N$ clients with local datasets $\{\mathcal{D}_i\}_{i=1}^{N}$ and a shared frozen LLM backbone $W_0$.
The trainable model comprises LoRA experts $\Theta=\{\theta_m\}_{m=1}^{M^\star}$ and a global router $\varphi$.
We call a coherent local sub-distribution that induces similar adaptation behavior a \emph{latent instruction pattern}.
Task labels are used only to construct and audit the benchmark; they are not provided to \method{}.

Figure~\ref{fig:fedweave_framework} summarizes the pipeline.
Each client first discovers local buckets $\{\mathcal{B}_{i,c}\}_{c=1}^{C_i}$, the server then aligns buckets into global expert groups, and local training interleaves bucket-homogeneous mini-batches.
The central asymmetry is that aligned-prototype sample counts weight expert updates to preserve pattern purity, whereas the global router spans each client's complete trajectory to retain cross-pattern contrast.
Implementation details, diagnostics, and proofs are in
Appendices~\ref{app:implementation}--\ref{app:theory}.

\subsection{Local Prototype Discovery}
Each client embeds its local examples with the frozen backbone $W_0$ and partitions $\mathcal{D}_i$ into local prototype buckets $\{\mathcal{B}_{i,c}\}_{c=1}^{C_i}$.
We call the conceptual specialization unit a \emph{prototype} and its concrete local implementation a \emph{bucket}.
This exposes intra-client structure without sharing raw examples.

Semantic proximity identifies local modes but does not determine whether their updates are compatible.
Each bucket therefore performs a short LoRA warmup from the same initialization and records an adaptation signature $s_{i,c}$ from the resulting $B$ matrices, following prior observations of functional asymmetry between the LoRA factors~\citep{guo2025selective,sun2024ffa}.
The signature enables cross-client matching of buckets that can share an expert while preserving expert-update purity.

\subsection{Cross-Client Prototype Alignment}
The alignment must decide both how many specialization units exist and which units should share parameters across clients.
Each client runs $k$-means on frozen-backbone embeddings and selects $C_i$ by maximizing the average silhouette coefficient~\citep{rousseeuw1987silhouettes}.
The server forms $D=(d_{uv})$ from the mean layer-wise cosine distance across $L_B$ LoRA-$B$ blocks.
Let $\mathrm{Agg}_M(D)$ denote agglomerative clustering into $M$ groups, where $\mathcal K_g$ is the feasible candidate set. We select $M\in\mathcal K_g$ by silhouette score:
\begin{equation}
\begin{aligned}
 d_{uv}&=\frac{1}{L_B}\sum_{\ell=1}^{L_B}\!\left[1-\cos(s_u^{(\ell)},s_v^{(\ell)})\right],\\
 M^\star&=\arg\max_{M\in\mathcal K_g}\mathrm{Sil}(\mathrm{Agg}_M(D)).
\end{aligned}
\label{eq:signature_alignment}
\end{equation}
The one-time alignment fixes $M^\star$ and $a(i,c)\in\{1,\ldots,M^\star\}$ throughout federated training, mapping bucket $c$ of client $i$ to one global expert.
With $\mathcal U_m=\{(i,c):a(i,c)=m\}$ and the warmup adapter $\widetilde\theta_{i,c}$, the implementation initializes
\begin{equation}
 \theta_m^0=\frac{1}{|\mathcal U_m|}\sum_{(i,c)\in\mathcal U_m}\widetilde\theta_{i,c}.
\label{eq:expert_initialization}
\end{equation}
Semantic similarity exposes local modes, while adaptation similarity determines which modes share an expert without benchmark task labels.
K-means and agglomerative clustering are default implementations; the framework requires only coherent local buckets and cross-client alignment, and Table~\ref{tab:ablation_alignment_quality} evaluates alternatives.

\subsection{Local Training over Interleaved Buckets}
After prototype alignment, each communication round trains experts and the router with different local granularities.
For an example $x\in\mathcal{B}_{i,c}$, 
the router produces $r_m(x;\varphi)\geq0$ with $\sum_m r_m(x;\varphi)=1$, and the bucket objective is
\begin{equation}
\mathcal{L}_{i,c}(\Theta,\varphi)=\frac{1}{|\mathcal{B}_{i,c}|}
\sum_{(x,y)\in\mathcal{B}_{i,c}}
\ell\!\left(f_{W_0,\Theta,r(x;\varphi)}(x),y\right).
\label{eq:bucket_objective}
\end{equation}
Soft routing keeps router-expert training differentiable and supplies comparative signals to non-anchor experts while early assignments remain uncertain, reducing premature expert starvation.
It can also transmit off-pattern gradients; the theory records this routing leakage in each expert's off-target contribution weight. 

Each client retains the same fixed $E$-step mini-batch budget, allocates it across buckets approximately in proportion to bucket size with a small minimum for every nonempty bucket, and then interleaves the scheduled bucket identifiers.
At each scheduled step, the client draws a mini-batch only from the selected bucket, preserving pattern-coherent expert supervision.
The broadcast global-router state persists across the entire interleaved schedule, accumulating cross-pattern contrast across the client's buckets.

\subsection{Asymmetric Federated Aggregation}
Let $\Delta\theta_{i,m}^{t}$ and $\Delta\varphi_i^t$ denote client $i$'s local expert and router deltas at round $t$, and let
$n_{i,m}=\sum_{c:a(i,c)=m}|\mathcal{B}_{i,c}|$ be the local sample count aligned to expert $m$ for FedAvg-style weighting~\citep{mcmahan2017fedavg}.
The server updates
\begin{equation}
\begin{aligned}
\theta_m^{t+1}&=\theta_m^t+\sum_{i\in\mathcal{S}_t:n_{i,m}>0}
\frac{n_{i,m}}{\sum_{j\in\mathcal{S}_t} n_{j,m}}\Delta\theta_{i,m}^{t},\\
\varphi^{t+1}&=\varphi^t+\sum_{i\in\mathcal{S}_t}
\frac{|\mathcal{D}_i|}{\sum_{j\in\mathcal{S}_t}|\mathcal{D}_j|}\Delta\varphi_i^t.
\end{aligned}
\label{eq:asymmetric_aggregation}
\end{equation}

The expert rule weights each delta by matched-pattern supervision to preserve expert-update purity, whereas the router rule contributes one delta from each client's full interleaved trajectory to retain cross-bucket contrast. 

\begin{algorithm}[t]
\caption{\method{} training and sparse inference}
\label{alg:fedweave_main}
\small
\algsetup{linenosize=\small}
\begin{algorithmic}[1]
\REQUIRE Client datasets $\{\mathcal{D}_i\}$, frozen backbone $W_0$, rounds $T$, local-step budget $E$
\FORALL{clients $i$}
    \STATE Embed $\mathcal{D}_i$, select $C_i$, and form buckets $\{\mathcal{B}_{i,c}\}$
    \STATE Warm up one LoRA per bucket and send signatures $\{s_{i,c}\}$
\ENDFOR
\STATE Cluster signatures into $M^\star$ experts and obtain $a(i,c)$
\STATE Initialize global experts $\{\theta_m^0\}$ and router $\varphi^0$
\FOR{$t=0$ to $T-1$}
    \STATE Sample $\mathcal{S}_t$; broadcast $\{\theta_m^t\}$, $\varphi^t$, and assignments
    \FORALL{$i\in\mathcal{S}_t$ in parallel}
        \STATE Allocate $E$ steps by bucket sample counts and interleave bucket IDs
        \STATE Update soft-routed experts and the broadcast global-router state
        \STATE Return $\Delta\theta_{i,m}$ when $n_{i,m}>0$ and one $\Delta\varphi_i$
    \ENDFOR
    \STATE Aggregate expert and router deltas using Eq.~\eqref{eq:asymmetric_aggregation}
\ENDFOR
\STATE At inference, activate $\theta_{\arg\max_m r_m(x;\varphi)}$ for each input $x$
\end{algorithmic}
\end{algorithm}

\subsection{Sparse Inference}
Training uses the full softmax mixture so that each example supplies comparative supervision across experts.
At inference, \method{} selects
$\widehat m(x)=\arg\max_m r_m(x;\varphi)$ and activates only $\theta_{\widehat m(x)}$; the analysis below connects router optimization to the risk of this sparse decision.

\subsection{Theory of Asymmetric Aggregation}
The analysis formalizes three advantages of \method{}.
First, prototype conditioning gives $O(1/U)$ expert convergence to a stationary neighborhood controlled by routed off-pattern weight.
Second, persistent router training contracts under the PL condition while exposing deterministic interleaving bias instead of incurring the reset consensus residual.
Third, gap calibration transfers router suboptimality to Top-1 sparse-inference risk.
The full assumptions, the mixed-gradient conflict calculation, and all proofs are given in Appendix~\ref{app:theory}.

\paragraph{Expert convergence: stationarity under routed contamination.}
The first advantage is expert convergence: prototype-conditioned aggregation makes the deviation from an expert's target gradient depend on the off-pattern weight routed to that expert, rather than the full client mixture.
To quantify it, fix the router and the other experts, let expert $m$ target pattern $\kappa_m$, and let
$\chi_m(\theta_m)=\sum_{i,c,k\neq\kappa_m}\omega_{i,c,k,m}(\theta_m)$
be its off-target routed contribution weight, where $\omega_{i,c,k,m}$ is the normalized weight of pattern-$k$ examples in client $i$'s bucket $c$ routed to expert $m$.
Under bounded routed gradients and within-pattern shift $\delta_{\mathrm{pat}}$, the directional error is at most
$\varepsilon_m^u=2B_g\chi_m(\theta_m^u)+\delta_{\mathrm{pat}}$.
Here $F_{\kappa_m,m}$ is expert $m$'s clean loss on pattern $\kappa_m$; assume it is $L_E$-smooth and lower bounded, and that the routed stochastic direction has conditional variance at most $(\rho_m^u)^2\sigma_E^2$.
Let $g_m^u=\nabla F_{\kappa_m,m}(\theta_m^u)$,
$\rho_m^u=\rho_m(\theta_m^u)$,
$a_m^u=\eta_E\rho_m^u\leq1/L_E$,
$A_{m,U}=\sum_{u<U}\mathbb E[a_m^u]>0$, and
$\Delta_{E,m}^0=F_{\kappa_m,m}(\theta_m^0)-F_{\kappa_m,m}^\star$.
\begin{theorem}[Expert convergence under routed contamination]
\label{thm:expert-convergence-main}
\begin{equation}
\begin{aligned}
\tfrac{1}{A_{m,U}}\!\sum_{u<U}\mathbb E[a_m^u\|g_m^u\|^2]
&\leq\tfrac{2\Delta_{E,m}^0}{A_{m,U}}
+\tfrac{1}{A_{m,U}}\!\sum_{u<U}\mathbb E[a_m^u(\varepsilon_m^u)^2]\\
&\quad+\tfrac{L_E\sigma_E^2}{A_{m,U}}\!\sum_{u<U}\mathbb E[(a_m^u)^2].
\end{aligned}
\label{eq:biased_stationary_main}
\end{equation}
\end{theorem}
With non-vanishing routed weight and bounded contamination, this is $O(1/U)$ convergence to a contamination--noise stationary neighborhood; $A_{m,U}$ records the effective optimization budget received by expert $m$ and thus exposes expert starvation.
With other terms matched, the prototype bound is no looser if $\chi_m(\theta_m^u)\leq\alpha_m^{\mathrm{client},u}$ at every active step, and strictly tighter if the inequality is strict with positive probability.

\paragraph{Router convergence: persistent versus reset-and-average.}
The second advantage is router convergence: FedWeave keeps one broadcast router state across interleaved buckets, so all bucket updates follow a single trajectory.
For comparison, reset-and-average maintains one state per bucket and averages them at the end, mirroring the router-state reset in the \emph{w/o client-router} ablation.
With experts fixed, let $R_{i,c}$ be the router loss on bucket $c$, and let
$R_i(\varphi)=\sum_c\pi_{i,c}R_{i,c}(\varphi)$,
$g_i(\varphi)=\nabla R_i(\varphi)$, and
$g_{i,c}(\varphi)=\nabla R_{i,c}(\varphi)$,
where $\pi_{i,c}=|\mathcal B_{i,c}|/|\mathcal D_i|$ is the bucket sample fraction.
Assume $L_R$-smooth bucket losses, a $\mu$-PL client objective~\citep{karimi2016linear}, conditional variance at most $\sigma_R^2$, and $0<\eta_R\leq1/L_R$.
Over $S$ local steps, let $\widehat g_{i,c}$ be the stochastic gradient on bucket $c$; for the persistent schedule and reset states, define
$q_s(\varphi_s)=\mathbb E[\widehat g_{i,c_s}(\varphi_s)\mid\varphi_s]-g_i(\varphi_s)$,
$\bar\varphi_s=\sum_c\pi_{i,c}\varphi_{c,s}$, and
$e_s=\sum_c\pi_{i,c}[g_{i,c}(\varphi_{c,s})-g_{i,c}(\bar\varphi_s)]$.
Let
\begin{equation}
\begin{aligned}
\mathcal E_{i,S}&=(1-\mu\eta_R)^S(R_i(\varphi_0)-R_i^\star)
+\tfrac{L_R\eta_R\sigma_R^2}{2\mu},\\
\mathcal O_{i,S}&=\tfrac{\eta_R}{2}\sum_{s<S}(1-\mu\eta_R)^{S-1-s}\mathbb E\|q_s(\varphi_s)\|^2,\\
\mathcal C_{i,S}&=\tfrac{\eta_R}{2}\sum_{s<S}(1-\mu\eta_R)^{S-1-s}\mathbb E\|e_s\|^2 .
\end{aligned}
\label{eq:router_residuals_main}
\end{equation}
\begin{theorem}[Persistent versus reset-and-average router convergence]
\label{thm:router-convergence-main}
Under matched initialization, depth, and effective pattern weights,
\begin{equation}
\begin{aligned}
\mathbb E[R_i(\varphi_S)-R_i^\star]&\leq\mathcal E_{i,S}+\mathcal O_{i,S},\\
\mathbb E[R_i(\bar\varphi_S)-R_i^\star]&\leq\mathcal E_{i,S}+\mathcal C_{i,S}.
\end{aligned}
\label{eq:persistent_router_main}
\end{equation}
\end{theorem}
Both procedures share the same PL contraction and stochastic-noise floor.
Sample-count-proportional randomized bucket sampling gives $\mathcal O_{i,S}=0$; \method{}'s deterministic interleaving incurs $\mathcal O_{i,S}$, whereas reset-and-average incurs $\mathcal C_{i,S}$ from bucket-local dispersion.

\paragraph{Sparse inference: router error controls Top-1 risk.}
Finally, although training uses a soft mixture, inference activates only one expert; insufficient router probability on the latent oracle expert can therefore incur excess loss.
In the clean one-pattern-per-expert regime, let $m^\dagger(x)$ be the oracle expert, $R_{i,\mathrm{oracle}}$ its risk, and
$S_i(\varphi)=\mathbb E[1-r_{m^\dagger(x)}(x;\varphi)]$.
Assume the expert-loss-gap calibration
$c_{\mathrm{gap}}S_i(\varphi)\leq R_i(\varphi)-R_{i,\mathrm{oracle}}+\beta_{\mathrm{mix}}$
with $c_{\mathrm{gap}}>0$, and let
$A_i=[R_i^\star-R_{i,\mathrm{oracle}}]_+$.
If non-oracle excess loss is bounded by $\Delta_{\max}$, then the Top-1 risk $R_{\mathrm{sparse},i}$ obeys
\begin{equation}
R_{\mathrm{sparse},i}-R_{i,\mathrm{oracle}}
\leq\frac{2\Delta_{\max}}{c_{\mathrm{gap}}}
[R_i(\varphi)-R_i^\star+A_i+\beta_{\mathrm{mix}}].
\label{eq:sparse_risk_main}
\end{equation}
Taking expectations, Theorem~\ref{thm:router-convergence-main} supplies $\mathcal E_{i,S}+\mathcal O_{i,S}$ for $\varphi_S$ and $\mathcal E_{i,S}+\mathcal C_{i,S}$ for $\bar\varphi_S$.

\begin{table*}[t]
\centering
\maintablesetup
\begin{tabularx}{0.94\textwidth}{@{}p{\mainmodelcol}lYYYYY@{}}
\toprule
\multicolumn{1}{l}{Model}
& \multicolumn{1}{l}{Method}
& \multicolumn{1}{c}{Macro Avg.}
& \multicolumn{1}{c}{CoEdIT}
& \multicolumn{1}{c}{GSM8K}
& \multicolumn{1}{c}{TweetEval}
& \multicolumn{1}{c}{ARC-C} \\
\midrule
\panelhead{7}{Task Scores $\uparrow$} \\
\midrule
\multirow{5}{*}{\itshape Llama3.2-3B}
& FedIT
& \secondres{0.5673}{0.0051}
& \bestres{0.6293}{0.0090}
& \res{0.3033}{0.0138}
& \res{0.6633}{0.0194}
& \res{0.6733}{0.0104} \\
& FFA-LoRA
& \res{0.5050}{0.0028}
& \res{0.5368}{0.0114}
& \res{0.2392}{0.0240}
& \res{0.6300}{0.0180}
& \res{0.6142}{0.0161} \\
& \mbox{FedSA-LoRA}
& \res{0.5558}{0.0043}
& \res{0.5999}{0.0048}
& \res{0.3067}{0.0095}
& \res{0.6433}{0.0243}
& \res{0.6733}{0.0063} \\
& FedLEASE
& \res{0.5649}{0.0090}
& \res{0.5948}{0.0537}
& \secondres{0.3142}{0.0161}
& \secondres{0.6742}{0.0236}
& \secondres{0.6767}{0.0052} \\
\fwrow
{\fwbestres{0.5872}{0.0047}}
{\fwsecondres{0.6190}{0.0067}}
{\fwbestres{0.3283}{0.0076}}
{\fwbestres{0.6875}{0.0225}}
{\fwbestres{0.7142}{0.0104}}
\midrule
\multirow{5}{*}{\itshape Gemma-2-2B}
& FedIT
& \secondres{0.4839}{0.0174}
& \secondres{0.3340}{0.0054}
& \res{0.3033}{0.0113}
& \secondres{0.6592}{0.0379}
& \secondres{0.6392}{0.0250} \\
& FFA-LoRA
& \res{0.4544}{0.0071}
& \res{0.3276}{0.0051}
& \res{0.2642}{0.0236}
& \res{0.6492}{0.0300}
& \res{0.5767}{0.0281} \\
& \mbox{FedSA-LoRA}
& \res{0.4664}{0.0081}
& \res{0.3297}{0.0055}
& \secondres{0.3108}{0.0163}
& \res{0.6225}{0.0139}
& \res{0.6025}{0.0261} \\
& FedLEASE
& \res{0.4755}{0.0074}
& \res{0.3313}{0.0131}
& \res{0.2942}{0.0208}
& \res{0.6408}{0.0456}
& \res{0.6358}{0.0118} \\
\fwrow
{\fwbestres{0.5163}{0.0154}}
{\fwbestres{0.3442}{0.0087}}
{\fwbestres{0.3333}{0.0456}}
{\fwbestres{0.6892}{0.0445}}
{\fwbestres{0.6983}{0.0080}}
\midrule
\panelhead{7}{Teacher-Forced Losses $\downarrow$} \\
\midrule
\multirow{5}{*}{\itshape Llama3.2-3B}
& FedIT
& \secondres{0.7496}{0.0095}
& \res{0.8537}{0.0144}
& \res{0.6258}{0.0054}
& \res{0.7070}{0.0355}
& \secondres{0.8120}{0.0244} \\
& FFA-LoRA
& \res{0.9113}{0.0084}
& \res{1.0396}{0.0064}
& \res{0.8720}{0.0139}
& \res{0.8110}{0.0272}
& \res{0.9226}{0.0027} \\
& \mbox{FedSA-LoRA}
& \res{0.8214}{0.0130}
& \res{0.9457}{0.0044}
& \res{0.6874}{0.0268}
& \res{0.7525}{0.0169}
& \res{0.9000}{0.0126} \\
& FedLEASE
& \res{0.7600}{0.0170}
& \secondres{0.8490}{0.0328}
& \secondres{0.6080}{0.0062}
& \secondres{0.7050}{0.0353}
& \res{0.8782}{0.0531} \\
\fwrow
{\fwbestres{0.7264}{0.0076}}
{\fwbestres{0.8046}{0.0071}}
{\fwbestres{0.6057}{0.0097}}
{\fwbestres{0.6931}{0.0385}}
{\fwbestres{0.8023}{0.0214}}
\midrule
\multirow{5}{*}{\itshape Gemma-2-2B}
& FedIT
& \secondres{0.7443}{0.0166}
& \secondres{0.8469}{0.0134}
& \secondres{0.5062}{0.0015}
& \secondres{0.7181}{0.0212}
& \secondres{0.9060}{0.0627} \\
& FFA-LoRA
& \res{0.8174}{0.0131}
& \res{0.9180}{0.0091}
& \res{0.5799}{0.0026}
& \res{0.7621}{0.0213}
& \res{1.0096}{0.0515} \\
& \mbox{FedSA-LoRA}
& \res{0.8255}{0.0020}
& \res{0.9135}{0.0053}
& \res{0.5613}{0.0279}
& \res{0.7821}{0.0275}
& \res{1.0452}{0.0496} \\
& FedLEASE
& \res{0.8030}{0.0155}
& \res{0.8677}{0.0363}
& \res{0.5184}{0.0237}
& \res{0.8372}{0.0353}
& \res{0.9888}{0.0441} \\
\fwrow
{\fwbestres{0.6954}{0.0132}}
{\fwbestres{0.7881}{0.0188}}
{\fwbestres{0.4915}{0.0034}}
{\fwbestres{0.6802}{0.0338}}
{\fwbestres{0.8218}{0.0380}}
\bottomrule
\end{tabularx}
\caption{Main results at $\alpha=0.3$ (mean $\pm$ standard deviation over seeds 42/43/44). ``Macro Avg.'' is the unweighted mean across tasks within each panel. Best and second-best results are \textbf{bold} and \underline{underlined}.}
\label{tab:main_results}
\end{table*}

\section{Experiments}
\label{sec:experiments}

\subsection{Experimental Setup}
\paragraph{Benchmark and Datasets.}
Our controlled heterogeneous multi-task benchmark combines CoEdIT~\citep{raheja2023coedit}, GSM8K~\citep{cobbe2021gsm8k}, TweetEval sentiment classification (TweetEval)~\citep{barbieri2020tweeteval}, and ARC-Challenge (ARC-C)~\citep{clark2018arc}.
For each seed, we sample disjoint per-task subsets of 2,000/100/400 train/validation/test examples from the pooled sources, yielding 1,600 test examples in total.
Equal task budgets keep client mixtures and macro-averages comparable across tasks.

\paragraph{Federated training.}
We partition training data among 20 clients by task-identity Dirichlet sampling; smaller $\alpha$ means stronger client imbalance.
The main setting is $\alpha=0.3$; mean $\pm$ standard deviation over seeds 42/43/44 jointly reflects data subsampling, client partitioning, and training randomness.
Task labels construct and audit this controlled benchmark but are never exposed to \method{}.
The benchmark contains distinct instruction patterns that can be recovered from the frozen representations; Table~\ref{tab:ablation_alignment_quality} reports the resulting alignment quality.
We evaluate Llama3.2-3B~\citep{meta2024llama3p2} and Gemma-2-2B~\citep{gemma2024gemma2} with all 20 clients participating for 20 rounds and 10 local optimizer steps per round.
All methods use the same local-step budget; \method{} redistributes its 10 steps across buckets according to their sample counts.
All runs use AdamW with learning rate 1e-4, effective batch size 8, gradient clipping at 1.0, and maximum length 512.
Single-adapter baselines use one rank-32 LoRA, whereas \method{} and FedLEASE use rank 8 per expert; all LoRAs use scaling factor 16, dropout 0.05, and query/value projections.
\method{} and FedLEASE select the expert count adaptively by silhouette search over prototype and client signatures, respectively; the reported main runs selected four experts for both methods.
For either backbone, this realized count matches the total LoRA capacity of its rank-32 single-adapter counterpart.
\method{} trains a hidden-size-512 global router with learning rate 5e-5, zero dropout.
Experts average their aligned warmup adapters, whereas the two-layer router uses default random initialization and initially distributes weight across all experts via soft routing.
At test time, sparse inference activates one expert.

\paragraph{Evaluation.}
CoEdIT uses mean ROUGE-1/2/Lsum F1; GSM8K and ARC-C use exact-answer accuracy; and TweetEval uses accuracy.
Greedy decoding permits 64/192/4/4 new tokens for CoEdIT/GSM8K/TweetEval/ARC-C; full settings are in Appendix~\ref{app:implementation}.

\subsection{Baselines and Comparison Scope}
Holding the backbone, LoRA capacity, data, and global-model target fixed, we compare four representative federated LoRA approaches.
\textbf{FedIT} jointly trains both factors of one shared LoRA and aggregates client deltas by sample-weighted FedAvg.
\textbf{FFA-LoRA} freezes the randomly initialized $A$ factor and trains and aggregates only the zero-initialized $B$ factor.
\textbf{FedSA-LoRA} aggregates $A$ globally while retaining client-local $B$ states across rounds; global evaluation pairs the shared $A$ with their sample-weighted mean.
\textbf{FedLEASE}, the closest empirical baseline, selects the client-cluster count by silhouette score and maintains one expert and router for each selected cluster.
Backbone-specific trainable-parameter counts are tabulated in Table~\ref{tab:parameter_budget_supp} of Appendix~\ref{app:implementation}.
\method{} remains in the same trainable-parameter range as the full-adapter baselines FedIT/FedSA and smaller than FedLEASE.
\subsection{Overall Effectiveness}

Table~\ref{tab:main_results} shows that, at $\alpha=0.3$, \method{} raises the best-baseline macro-average score from 0.5673 to 0.5872 on Llama3.2-3B and from 0.4839 to 0.5163 on Gemma-2-2B, while reducing macro-average loss from 0.7496 to 0.7264 and from 0.7443 to 0.6954, respectively.
On Llama3.2-3B it leads on GSM8K, TweetEval, and ARC-C, while FedIT remains strongest on CoEdIT; on Gemma-2-2B it ranks first on all four tasks.
The gain therefore persists across two architectures without requiring a task-specific expert assignment.
Figure~\ref{fig:routing_heatmap} shows that \method{} assigns the four tasks to distinct dominant experts (E1/E4/E3/E2), whereas FedLEASE reuses E3 for CoEdIT and GSM8K and distributes the remaining tasks more broadly.

\begin{figure}[t]
    \centering
    \includegraphics[width=0.90\linewidth]{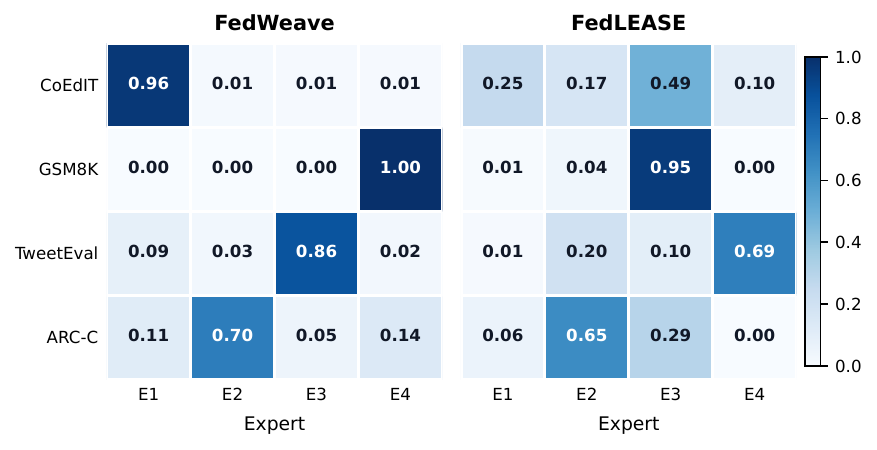}
    \caption{Row-normalized task--expert routing weights on Llama3.2-3B at $\alpha=0.3$. \method{} learns a distinct dominant expert for each task compared to FedLEASE.}
    \label{fig:routing_heatmap}
\end{figure}

\subsection{Robustness to Client Heterogeneity}
Figure~\ref{fig:trend_heterogeneity} shows that, against FedLEASE on Llama3.2-3B, \method{} improves macro-average score by 0.0150, 0.0223, and 0.0106 for $\alpha=0.1,0.3,0.5$, respectively, while reducing macro-average loss by 0.0256, 0.0336, and 0.0245.
The advantage remains positive at all three heterogeneity levels and is largest at $\alpha=0.3$.

\begin{figure}[t]
\centering
\includegraphics[width=\linewidth]{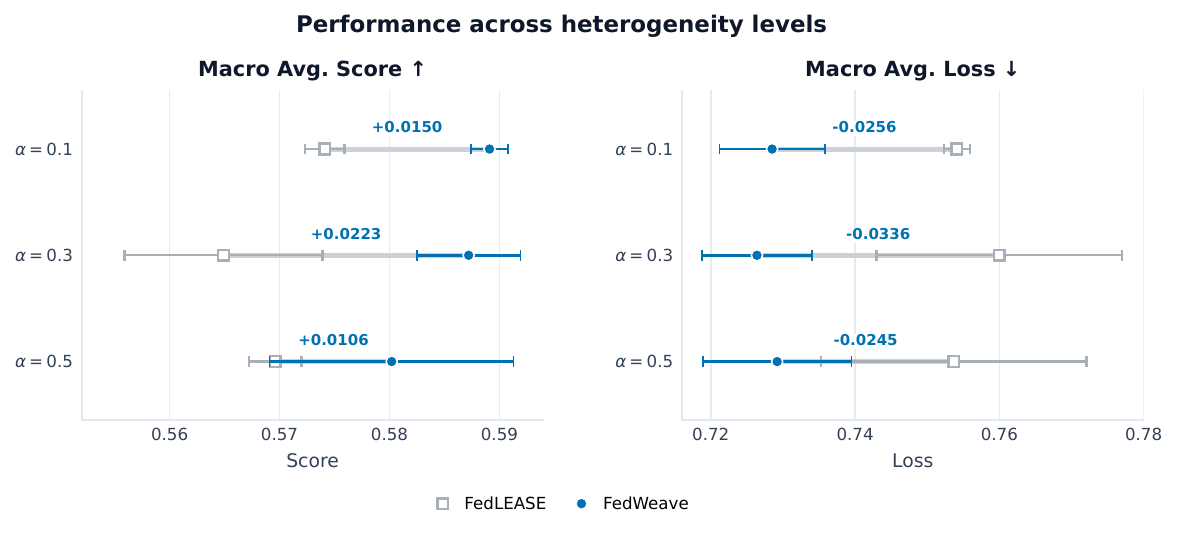}
\caption{Macro-average score and teacher-forced loss on Llama3.2-3B across Dirichlet heterogeneity levels. Shaded regions show one standard deviation over three seeds.}
\label{fig:trend_heterogeneity}
\end{figure}

\subsection{Ablation Study}
\subsubsection{Aggregation Granularity}
With the same backbone, data, and expert capacity, \emph{w/o prototype-expert} uses Client/Client aggregation, whereas \emph{w/o client-router} uses Prototype/Prototype aggregation with per-bucket router resets and sample-weighted endpoints.

\begin{table*}[t]
\centering
\captionsetup{font=small,skip=2pt}
\small
\setlength{\tabcolsep}{2.5pt}
\renewcommand{\arraystretch}{0.90}
\begin{tabularx}{0.81\textwidth}{@{}>{\raggedright\arraybackslash}p{0.18\textwidth}>{\centering\arraybackslash}p{0.095\textwidth}>{\centering\arraybackslash}p{0.095\textwidth}|YYYYY@{}}
\toprule
Method & Expert & Router & Average & CoEdIT & GSM8K & TweetEval & ARC-C \\
\midrule
\multicolumn{8}{@{}l}{\textit{Task Metrics} $\uparrow$} \\
\midrule
w/o prototype-expert & Client & Client & 0.5649 & 0.5948 & 0.3142 & 0.6742 & 0.6767 \\
w/o client-router & Prototype & Prototype & 0.5578 & 0.6013 & 0.3225 & 0.6408 & 0.6667 \\
\textbf{\method{}} & \textbf{Prototype} & \textbf{Client} & \textbf{0.5872} & \textbf{0.6190} & \textbf{0.3283} & \textbf{0.6875} & \textbf{0.7142} \\
\midrule
\multicolumn{8}{@{}l}{\textit{Task Losses} $\downarrow$} \\
\midrule
w/o prototype-expert & Client & Client & 0.7600 & 0.8490 & 0.6080 & 0.7050 & 0.8782 \\
w/o client-router & Prototype & Prototype & 0.7692 & 0.8066 & 0.6071 & 0.7622 & 0.9008 \\
\textbf{\method{}} & \textbf{Prototype} & \textbf{Client} & \textbf{0.7264} & \textbf{0.8046} & \textbf{0.6057} & \textbf{0.6931} & \textbf{0.8023} \\
\bottomrule
\end{tabularx}
\caption{Aggregation-granularity ablation. \emph{w/o prototype-expert}: Client/Client; \emph{w/o client-router}: Prototype/Prototype with per-bucket router resets and sample-weighted endpoints; \method{}: Prototype/Client with one global router following complete client trajectories.}
\label{tab:ablation_asymmetry}
\vspace{-4pt}
\end{table*}

\begin{table*}[t]
\centering
\captionsetup{font=small,skip=2pt}
\small
\setlength{\tabcolsep}{3.0pt}
\renewcommand{\arraystretch}{0.90}
\begin{tabular*}{0.90\textwidth}{@{\extracolsep{\fill}}lcccc|ccc|cc@{}}
\toprule
\multirow{2}{*}{Local Clusterer} & \multirow{2}{*}{Complexity} & \multicolumn{3}{c|}{Local Alignment} & \multicolumn{3}{c|}{Global Alignment} & \multirow{2}{*}{Macro Score $\uparrow$} & \multirow{2}{*}{Macro Loss $\downarrow$} \\
\cmidrule(lr){3-5}\cmidrule(lr){6-8}
& & Purity $\uparrow$ & NMI $\uparrow$ & ARI $\uparrow$ & Purity $\uparrow$ & NMI $\uparrow$ & ARI $\uparrow$ & & \\
\midrule
K-Means & $O(I n_i C_i d_e)$ & 0.9685 & 0.9051 & 0.9288 & 0.9515 & 0.8567 & 0.8763 & 0.5872 & 0.7264 \\
Agglomerative & $O(n_i^2 \log n_i)$ & 0.7680 & 0.2635 & 0.2669 & 0.2522 & 0.0044 & 0.0000 & 0.5684 & 0.7526 \\
Spectral & $O(n_i^3)$ & 0.9728 & 0.9000 & 0.9013 & 0.9499 & 0.8649 & 0.8780 & \textbf{0.5889} & \textbf{0.7242} \\
\bottomrule
\end{tabular*}
\caption{Prototype alignment quality under different local clusterers. Purity measures dominant-label agreement, NMI measures normalized information agreement, and ARI measures chance-corrected pairwise agreement; exact definitions are in Appendix~\ref{app:alignment_metrics}. Complexities are per client, where $n_i=|\mathcal D_i|$, $C_i$ is the selected local bucket count, $d_e$ is the frozen embedding dimension, and $I$ is the number of $k$-means iterations; server-side signature alignment is fixed.}
\label{tab:ablation_alignment_quality}
\vspace{-4pt}
\end{table*}

Table~\ref{tab:ablation_asymmetry} shows macro-average gains of 0.0223 and 0.0294 over \emph{w/o prototype-expert} and \emph{w/o client-router}, with loss reductions of 0.0336 and 0.0428.
\method{} leads every task score and loss; against \emph{w/o client-router}, its largest score gains on TweetEval and ARC-C support preserving cross-bucket contrast within one router trajectory.
The \emph{w/o prototype-expert} comparison supports forming experts from coherent bucket contexts.
Because the \emph{w/o client-router} implementation has effective weight
$\widetilde\pi_c\propto a_cH_c$, where $a_c$ is the endpoint aggregation coefficient and $H_c$ is the reset depth for bucket $c$, its comparison with \method{} captures both router-state fragmentation and bucket reweighting, as predicted by the asymmetric analysis.
These comparisons support the component-dependent purity--contrast design.

\subsubsection{Prototype Discovery}
Table~\ref{tab:ablation_alignment_quality} compares client-side clusterers.
$k$-means and spectral clustering achieve nearly identical global purity (0.9515 versus 0.9499) and similar macro scores (0.5872 versus 0.5889).
However, $k$-means scales as $O(I n_i C_i d_e)$, whereas spectral clustering costs $O(n_i^3)$.
We therefore adopt $k$-means as the default because it preserves high-quality prototype discovery at substantially lower asymptotic cost.

\subsubsection{Sparse Inference}
Using the same trained checkpoints over three seeds, Figure~\ref{fig:ablation_sparse_inference} shows that sparse inference (activating the highest-scoring expert) reaches 0.5872 versus 0.5883 for soft routing and lowers mean latency from 3177 to 2147 ms (32.4\%).
Activating two experts reaches 0.5890 at 2587 ms, only 0.0018 above sparse inference at higher latency.

\begin{figure}[H]
\centering
\includegraphics[width=0.90\linewidth]{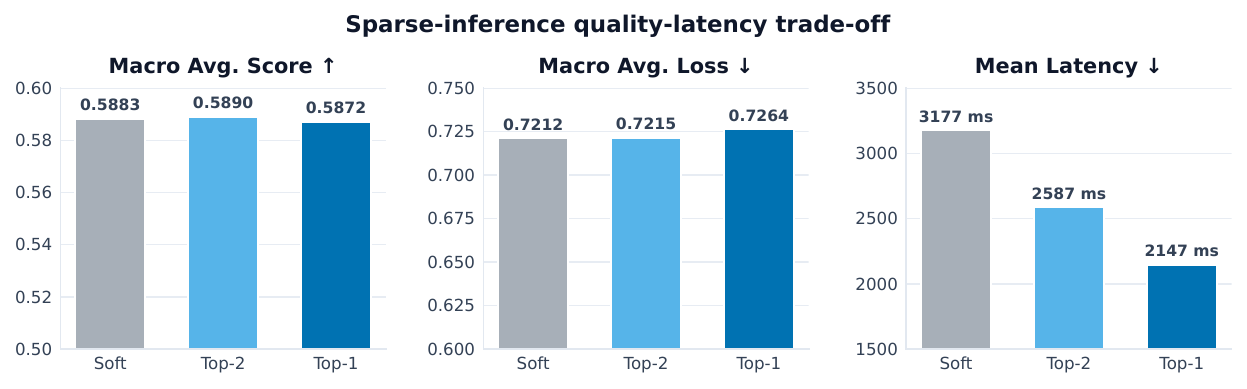}
\caption{Sparse-inference quality--latency trade-off over three seeds. Sparse inference activates the Top-1 expert at lower latency but nearly preserves soft-routing quality.}
\label{fig:ablation_sparse_inference}
\end{figure}

\section{Limitations}
\paragraph{Evaluation scope.}
Due to the limited computational resources, our study focuses on 2B--3B backbones and controlled four-task tests.
This scope enables matched comparisons across methods, backbones, and heterogeneity levels, while broader validation on larger models and naturally occurring mixtures remains future work.

\paragraph{Recoverable specialization.}
\method{} assumes that local patterns are sufficiently separable in representation and adaptation-signature space for prototype recovery and cross-client alignment.
The high alignment purity in Table~\ref{tab:ablation_alignment_quality} supports this condition in the studied regime, where pronounced task heterogeneity makes specialization most relevant.
When heterogeneity is weak and adaptation signatures become less distinguishable, the benefit of specialization may diminish.

\paragraph{Privacy and discovery overhead.}
\method{} currently does not incorporate differential privacy or secure aggregation for its update-derived signatures; integrating these mechanisms is an important extension.
Prototype discovery also introduces a one-time first-stage overhead from embedding, short LoRA warmups, and signature upload, which is amortized across subsequent communication rounds.
Detailed computation, communication, and parameter accounting is provided in Appendix~\ref{app:cost_accounting}.

\section{Conclusion}
\label{sec:conclusion}
\method{} addresses intra-client task heterogeneity by assigning specialization and routing to different aggregation units: aligned local prototypes form coherent expert updates, while a single global router trained from complete client trajectories retains cross-pattern contrast.
The component-wise theory mirrors this design: prototype conditioning controls routing-weighted expert contamination, persistent trajectories replace the reset consensus residual with an explicit schedule-order residual, and gap calibration connects router optimization to sparse-inference risk.
Across two backbones, it outperforms the evaluated baselines on the controlled benchmark, and its advantage over FedLEASE persists across three heterogeneity levels.
The granularity ablation supports the asymmetric design, while sparse inference preserves nearly all soft-routing quality with one active expert.
These results identify aggregation granularity, rather than expert count alone, as a central design choice for federated MoE-LoRA.
Future work should test natural, drifting streams with privacy-preserving online alignment.

\bibliography{references}

\appendix
\setcounter{secnumdepth}{1}
\small

\section{Implementation and Reproducibility Details}
\label{app:implementation}

Table~\ref{tab:technical_defaults} collects the resolved settings used by the reported runs.
Local discovery uses $\ell_2$-normalized, mean-pooled last-layer prompt representations from the frozen backbone.
For a client with $n$ examples, the candidate counts are $2,\ldots,\min(8,n-1)$; the default is $k$-means with 10 restarts, selected by Euclidean silhouette score.
Each bucket then warms up a rank-8 LoRA from the same initialization for 10 steps with batch size 8.
Its signature concatenates the LoRA-$B$ factors on query/value projections across layers.
The server computes the mean layer-wise cosine distance between signatures and applies silhouette-selected agglomerative clustering over the same feasible count range, so the number of global experts is data-adaptive; this procedure selects four experts in the reported main runs.
FedLEASE independently searches from 2 to at most 8 client clusters by the same silhouette criterion and likewise selects four experts in the reported runs.

All 20 clients participate in every communication round.
The router receives the attention-mask-weighted mean of the frozen backbone's last hidden states, computed with LoRA adapters disabled, and uses a two-layer MLP with hidden size 512 followed by a full softmax over experts.
The single global router is broadcast to each client, updated across its interleaved bucket schedule, and returned as one delta per round.
Generation uses greedy decoding (\texttt{do\_sample=False}, \texttt{num\_beams=1}).
The code package accompanying the paper contains training and evaluation scripts, deterministic data-construction code, and a manifest recording source dataset identifiers, split pools, sample budgets, and seed protocols.
When executed, the generated data cache records the source split and row index of every sampled example.

\begin{table*}[t]
\centering
\caption{Resolved implementation settings for the reported experiments. ``Effective batch'' counts gradient accumulation; only the memory-bound Gemma FedLEASE runs use micro-batch 4 with accumulation 2. Full-softmax training uses every selected expert, so the stored cumulative-probability threshold is inactive in these runs.}
\label{tab:technical_defaults}
\small
\setlength{\tabcolsep}{5pt}
\renewcommand{\arraystretch}{1.08}
\begin{tabularx}{0.96\textwidth}{@{}p{0.16\textwidth}p{0.22\textwidth}X@{}}
\toprule
Group & Setting & Resolved value \\
\midrule
Benchmark
& Data and partition
& CoEdIT, GSM8K, TweetEval sentiment, and ARC-Challenge; 2,000/100/400 train/validation/test examples per task; 20 clients; main Dirichlet $\alpha=0.3$; seeds 42/43/44. \\
\addlinespace[2pt]
Infrastructure
& Hardware and software
& One NVIDIA GeForce RTX 4090 (48 GB GPU memory; 49,140 MiB reported by \texttt{nvidia-smi}); Intel Xeon Platinum 8470Q (208 logical CPUs); 1.0 TiB system RAM; Ubuntu 22.04.5 LTS; NVIDIA driver 580.105.08 (CUDA compatibility 13.0); Python 3.12.3; PyTorch 2.8.0+\texttt{cu128} (CUDA 12.8, cuDNN 9.10.2); Transformers 5.0.0; PEFT 0.18.1; Datasets 3.6.0; scikit-learn 1.8.0; NumPy 2.3.2. \\
\addlinespace[2pt]
Federation
& Participation and budget
& Full participation; 20 rounds; 10 local optimizer steps per client and round. \\
\addlinespace[2pt]
Optimization
& AdamW
& Adapter learning rate $10^{-4}$; constant schedule; weight decay 0; gradient clipping 1.0; effective batch size 8. \\
\addlinespace[2pt]
LoRA
& Shared configuration
& Query/value projections; scaling factor 16; dropout 0.05; rank 32 for FedIT, FFA-LoRA, and FedSA-LoRA; rank 8 per adaptively selected expert for FedLEASE and \method{}. \\
\addlinespace[2pt]
Router
& Architecture and training
& Two-layer MLP; hidden size 512; dropout 0; learning rate $5\times10^{-5}$; full softmax over the selected experts; no task-label or load-balancing loss. \\
\addlinespace[2pt]
Discovery
& Local buckets
& Normalized frozen representations; $k$-means with 10 restarts; candidate count $2,\ldots,\min(8,n-1)$; Euclidean silhouette selection. \\
\addlinespace[2pt]
Discovery
& Signatures and alignment
& 10 warmup steps per bucket with batch size 8; embedding batch size 16; LoRA-$B$ query/value signatures; mean layer-wise cosine distance; silhouette-selected agglomerative alignment. \\
\addlinespace[2pt]
Evaluation
& Sequence and decoding
& Maximum input length 512; greedy decoding; 64/192/4/4 new tokens for CoEdIT/GSM8K/TweetEval/ARC-C; sparse inference (Top-1 expert activation). \\
\bottomrule
\end{tabularx}
\end{table*}

\begin{table*}[t]
\centering
\caption{Trainable-parameter budget for the two backbones. Counts exclude the frozen backbone, optimizer state, and replicated client-local copies. ``Adaptive (4)'' means that silhouette selection determined the expert count and selected four in the reported runs.}
\label{tab:parameter_budget_supp}
\small
\setlength{\tabcolsep}{6pt}
\renewcommand{\arraystretch}{1.08}
\begin{tabular}{@{}lccccc@{}}
\toprule
Method & LoRA rank & Experts & Router copies & \multicolumn{2}{c}{Trainable parameters (M)} \\
\cmidrule(lr){5-6}
& & & & Llama3.2-3B & Gemma-2-2B \\
\midrule
FedIT & 32 & 1 & -- & 9.18 & 6.39 \\
FFA-LoRA & 32 & 1 & -- & 3.67 & 2.56 \\
FedSA-LoRA & 32 & 1 & -- & 9.18 & 6.39 \\
FedLEASE & 8 & Adaptive & 4 cluster-specific & 15.48 & 11.12 \\
\method{} & 8 & Adaptive & 1 shared & 10.75 & 7.57 \\
\bottomrule
\end{tabular}
\end{table*}

\subsection{Alignment Metrics}
\label{app:alignment_metrics}

Let $\Omega=\{\omega_1,\ldots,\omega_R\}$ be the discovered clusters and
$\mathcal Y=\{y_1,\ldots,y_K\}$ the reference pattern labels used only for evaluation.
For $n$ evaluated examples or prototypes, purity is the fraction assigned consistently with the dominant reference label in each cluster:
\begin{equation}
\operatorname{Purity}(\Omega,\mathcal Y)
=\frac{1}{n}\sum_{r=1}^{R}\max_k|\omega_r\cap y_k|.
\label{eq:purity_definition}
\end{equation}
Normalized mutual information measures the information shared by the two partitions:
\begin{equation}
\operatorname{NMI}(\Omega,\mathcal Y)
=\frac{2I(\Omega;\mathcal Y)}{H(\Omega)+H(\mathcal Y)}.
\label{eq:nmi_definition}
\end{equation}
For $n_{rk}=|\omega_r\cap y_k|$, $a_r=\sum_kn_{rk}$, and $b_k=\sum_rn_{rk}$, the adjusted Rand index is
\begin{equation}
\operatorname{ARI}
=
\frac{
\sum_{r,k}\binom{n_{rk}}{2}
-
\dfrac{\sum_r\binom{a_r}{2}\sum_k\binom{b_k}{2}}{\binom{n}{2}}
}{
\dfrac{1}{2}\left[
\sum_r\binom{a_r}{2}+\sum_k\binom{b_k}{2}
\right]
-
\dfrac{\sum_r\binom{a_r}{2}\sum_k\binom{b_k}{2}}{\binom{n}{2}}
}.
\label{eq:ari_definition}
\end{equation}
Purity emphasizes dominant-label correctness, NMI measures global information agreement, and ARI measures pairwise agreement after correcting for chance.

\section{Cost Accounting}
\label{app:cost_accounting}

Prototype discovery runs once.
For client $i$, it embeds $|\mathcal{D}_i|$ examples, clusters them, and performs $10C_i$ bucket warmup steps.
Federated training schedules $20\times10=200$ local steps per client, so warmup alone adds a $C_i/20$ step ratio (20\% for four buckets), excluding embedding and clustering.
For $P=\sum_i C_i$ uploaded prototypes, the server computes pairwise layer-wise cosine distances and silhouette-based agglomerative groups.

For rank $r$ over $L$ layers, query/value outputs $o_{l,q},o_{l,v}$ give signature dimension $d_B=r\sum_{l=1}^{L}(o_{l,q}+o_{l,v})$.
For Llama3.2-3B at rank 8, $d_B=917{,}504$ FP32 values, or 3.50 MiB per bucket; client $i$ sends $3.50C_i$ MiB once.
During each communication round, \method{} and FedLEASE use rank 8 per adaptively selected expert.
For the four-expert selections on Llama3.2-3B, both instantiate 9.18M LoRA parameters, matching the total LoRA capacity of rank-32 FedIT.
Including routing modules, the corresponding worst-case trainable-state payload contains 10.75M scalars for \method{} and 15.48M for FedLEASE; Table~\ref{tab:parameter_budget_supp} reports the backbone-specific totals for both models.
Inactive expert uploads are omitted; bytes depend on precision.
These counts expose deterministic overhead but do not replace end-to-end wall-clock and peak-memory profiling.

\section{Theoretical Analysis}
\label{app:theory}
This section analyzes the two mechanisms behind \method{}.
We first analyze gradient conflict and routing-weighted expert convergence with the router and other experts fixed.
We then fix the experts, compare persistent and reset-and-average router optimization, and transfer the bounds to sparse-inference risk.
Randomized sample-count-proportional bucket sampling defines the clean router reference, while an explicit order-bias term covers deterministic balanced interleaving.
Task identities are used only as observable proxies for latent patterns in the analysis and diagnostics; \method{} never observes them.

\subsection{Gradient Conflict under Pattern Mixtures}
\label{app:gradient_conflict}

Let $F_k(\theta)$ be the expected loss of latent pattern $k$ for the component under study and let
$g_k(\theta)=\nabla F_k(\theta)$.
If client $i$ contains pattern proportions $p_{i,k}$, then its mixed objective and gradient are
\begin{equation}
F_i(\theta)=\sum_{k=1}^{K}p_{i,k}F_k(\theta),
\qquad
g_i(\theta)=\sum_{k=1}^{K}p_{i,k}g_k(\theta).
\label{eq:mixed_client_gradient}
\end{equation}
The squared norm of the mixed direction expands as
\begin{equation}
\|g_i\|^2
=
\sum_kp_{i,k}^2\|g_k\|^2
+2\sum_{k<\ell}p_{i,k}p_{i,\ell}
\langle g_k,g_\ell\rangle.
\label{eq:mixed_gradient_cancellation}
\end{equation}
Thus a negative cross-pattern inner product cancels part of the usable descent direction before server aggregation.
Define the corresponding conflict measure
\begin{equation}
\Gamma_i(\theta)
=
\sum_{k<\ell}p_{i,k}p_{i,\ell}
\bigl[-\langle g_k(\theta),g_\ell(\theta)\rangle\bigr]_+.
\label{eq:client_gradient_conflict}
\end{equation}

The role of local purity can be stated without assuming that every pattern pair conflicts.
For bucket $c$ of client $i$, let $q_{i,c,k}$ be its pattern proportions and $k(i,c)=\arg\max_kq_{i,c,k}$.
If its off-dominant probability is at most $\epsilon$ and $\|g_k\|\leq G$, then its conflict measure satisfies
\begin{equation}
\Gamma_{i,c}^{\mathrm{bucket}}
:=\sum_{k<\ell}q_{i,c,k}q_{i,c,\ell}
\bigl[-\langle g_k,g_\ell\rangle\bigr]_+
\leq
\left(\epsilon+\frac{\epsilon^2}{2}\right)G^2.
\label{eq:pure_bucket_conflict}
\end{equation}
Indeed, pairs involving the dominant pattern have total probability weight at most $\epsilon$, pairs among off-dominant patterns have weight at most $\epsilon^2/2$, and each negative inner product is at most $G^2$ in magnitude.
Equation~\eqref{eq:pure_bucket_conflict} shows that coherent buckets bound worst-case within-update cancellation through their off-dominant probability.
The reported purity scores diagnose this condition but do not directly measure gradient inner products.

\subsection{Routing-Weighted Expert Contamination}
\label{app:expert_contamination}

There are $K$ latent patterns and $M^\star$ global experts.
For the clean comparison, assume $K=M^\star$ and a bijection $\psi:\{1,\ldots,K\}\to\{1,\ldots,M^\star\}$, and write $\kappa_m=\psi^{-1}(m)$.
When $K\neq M^\star$, the fidelity bound still applies to experts with a designated target pattern after unmatched or merged pattern weight is included in the contamination term, but the one-to-one oracle-routing statements below require the clean regime.

For expert $m$, let $\rho_m(\theta_m)>0$ denote its total routed contribution weight and decompose the raw expected soft-routed gradient by client $i$, local bucket $c$, and latent pattern $k$:
\begin{equation}
\begin{aligned}
\smash{h_m^{\mathrm{raw}}(\theta_m)}
&=\rho_m(\theta_m)\bar h_m(\theta_m),\\
\bar h_m(\theta_m)
&=\sum_{i,c,k}\omega_{i,c,k,m}(\theta_m)
g_{i,c,k,m}(\theta_m),\\
\omega_{i,c,k,m}&\geq0,
\qquad \sum_{i,c,k}\omega_{i,c,k,m}=1.
\end{aligned}
\label{eq:routed_gradient_decomposition}
\end{equation}
We condition on positive total routed weight when defining $\bar h_m$; an inactive expert has zero raw update.
Here $\omega_{i,c,k,m}$ is the normalized nonnegative contribution weight after accounting for bucket sampling, prototype-conditioned upload weighting, and input-dependent routing gates.
The conditional direction $g_{i,c,k,m}$ retains the corresponding loss sensitivity, so Eq.~\eqref{eq:routed_gradient_decomposition} does not factor a mean routing coefficient from an unweighted bucket gradient.

\begin{assumption}[Bounded routed gradients and within-pattern shift]
\label{ass:routed-gradients}
For every $i,c,k,m$ and iterate,
\begin{equation}
\begin{aligned}
\|g_{i,c,k,m}(\theta_m)\|&\leq B_g,
&\|g_{k,m}(\theta_m)\|&\leq B_g,\\
\|g_{i,c,k,m}(\theta_m)-g_{k,m}(\theta_m)\|
&\leq\delta_{\mathrm{pat}}.
\end{aligned}
\label{eq:within_pattern_shift}
\end{equation}
where $g_{k,m}=\nabla F_{k,m}$ is the reference gradient for pattern $k$ and expert $m$.
\end{assumption}

Define the effective routed contamination of expert $m$ as
\begin{equation}
\chi_m(\theta_m)
=\sum_{i,c,k\neq\kappa_m}\omega_{i,c,k,m}(\theta_m).
\label{eq:routed_contamination}
\end{equation}
Let $k(i,c)$ be the dominant pattern of client $i$'s bucket $c$, and let $a(i,c)$ be its aligned anchor expert.
Using the same routed contribution weights, define
\begin{equation}
\begin{aligned}
\epsilon_m^{\mathrm{loc}}
&=\sum_{i,c,k}\omega_{i,c,k,m}\mathbf 1\{k\neq k(i,c)\},\\
\epsilon_m^{\mathrm{align}}
&=\sum_{i,c,k}\omega_{i,c,k,m}
\mathbf 1\{k(i,c)\neq\kappa_{a(i,c)}\},\\
\epsilon_m^{\mathrm{leak}}
&=\sum_{i,c,k}\omega_{i,c,k,m}\mathbf 1\{a(i,c)\neq m\}.
\end{aligned}
\label{eq:contamination_components}
\end{equation}

\begin{lemma}[Contamination decomposition]
\label{lem:contamination_decomposition}
In the clean one-pattern-per-expert regime,
\begin{equation}
\chi_m
\leq
\epsilon_m^{\mathrm{loc}}
+\epsilon_m^{\mathrm{align}}
+\epsilon_m^{\mathrm{leak}}.
\label{eq:contamination_decomposition}
\end{equation}
\end{lemma}

\begin{proof}
If $k\neq\kappa_m$, then at least one of the following must hold:
$k\neq k(i,c)$, $k(i,c)\neq\kappa_{a(i,c)}$, or $a(i,c)\neq m$.
Otherwise $k=k(i,c)=\kappa_{a(i,c)}$ and $a(i,c)=m$, which implies $k=\kappa_m$.
Applying this indicator union bound inside Eq.~\eqref{eq:routed_contamination} proves the result.
\end{proof}

The local and global purity values in the main text diagnose the first two terms, while routing concentration qualitatively diagnoses the third.
Equation~\eqref{eq:contamination_decomposition} connects these diagnostics to $\chi_m$; numerical equality would require additional assumptions.

\begin{proposition}[Routing-weighted expert fidelity]
\label{prop:routing-fidelity}
Under Assumption~\ref{ass:routed-gradients},
\begin{equation}
\|\bar h_m(\theta_m)-g_{\kappa_m,m}(\theta_m)\|
\leq
2B_g\chi_m(\theta_m)+\delta_{\mathrm{pat}}.
\label{eq:routed_expert_bias}
\end{equation}
\end{proposition}

\begin{proof}
Insert $g_{k,m}$ into Eq.~\eqref{eq:routed_gradient_decomposition}.
The weighted within-pattern shift is at most $\delta_{\mathrm{pat}}$.
For $k=\kappa_m$, the remaining difference is zero; for $k\neq\kappa_m$, boundedness gives
$\|g_{k,m}-g_{\kappa_m,m}\|\leq2B_g$.
Weighting the off-target terms proves Eq.~\eqref{eq:routed_expert_bias}.
\end{proof}

Combining Proposition~\ref{prop:routing-fidelity} with Lemma~\ref{lem:contamination_decomposition} yields the interpretable bound
\begin{equation}
\|\bar h_m-g_{\kappa_m,m}\|
\leq
2B_g\bigl(
\epsilon_m^{\mathrm{loc}}
+\epsilon_m^{\mathrm{align}}
+\epsilon_m^{\mathrm{leak}}
\bigr)
+\delta_{\mathrm{pat}}.
\label{eq:decomposed_expert_bias}
\end{equation}

For a client-level expert unit, define the analogous off-target contribution weight $\alpha_m^{\mathrm{client}}$ from its routing-weighted update decomposition.

\begin{corollary}[Direct granularity comparison]
\label{cor:direct-granularity}
The client-level direction obeys
\begin{equation}
\|\bar h_m^{\mathrm{client}}-g_{\kappa_m,m}\|
\leq2B_g\alpha_m^{\mathrm{client}}+\delta_{\mathrm{pat}}.
\label{eq:client_expert_bias}
\end{equation}
Consequently, prototype conditioning gives a strictly tighter worst-case target-direction bound whenever
$\chi_m<\alpha_m^{\mathrm{client}}$.
\end{corollary}

\begin{assumption}[Smooth target objective and routed stochastic direction]
\label{ass:expert-smoothness}
For expert $m$, $F_{\kappa_m,m}$ is $L_E$-smooth and lower bounded by $F_{\kappa_m,m}^\star$.
At active expert-update step $u$, let $\mathcal F_u$ contain $\theta_m^u$, the total routed weight $\rho_m^u>0$, and the routing history before stochastic gradient noise is drawn.
The raw stochastic direction satisfies
\begin{equation}
\begin{aligned}
\mathbb E[d_m^u\mid\mathcal F_u]
&=\rho_m^u\bigl(g_{\kappa_m,m}(\theta_m^u)+b_m^u\bigr),\\
\mathbb E[\|d_m^u-\mathbb E[d_m^u\mid\mathcal F_u]\|^2\mid\mathcal F_u]
&\leq(\rho_m^u)^2\sigma_E^2,
\end{aligned}
\end{equation}
where Proposition~\ref{prop:routing-fidelity} gives
$\|b_m^u\|\leq2B_g\chi_m(\theta_m^u)+\delta_{\mathrm{pat}}$.
Inactive steps have effective step size zero.
\end{assumption}

\begin{theorem}[Per-expert convergence in stationarity under routed contamination]
\label{thm:expert-stationarity-appendix}
Let $\theta_m^{u+1}=\theta_m^u-\eta_Ed_m^u$,
$a_m^u=\eta_E\rho_m^u\leq1/L_E$, and
$A_{m,U}=\sum_{u<U}\mathbb E[a_m^u]>0$.
Define $\Delta_{E,m}^0=F_{\kappa_m,m}(\theta_m^0)-F_{\kappa_m,m}^\star$.
Write $g_m^u=g_{\kappa_m,m}(\theta_m^u)$ and
$\varepsilon_m^u=2B_g\chi_m(\theta_m^u)+\delta_{\mathrm{pat}}$.
Then
\begin{equation}
\begin{aligned}
\frac{1}{A_{m,U}}\sum_{u<U}
\mathbb E\!\left[
a_m^u\|g_m^u\|^2
\right]
&\leq
\frac{2\Delta_{E,m}^0}{A_{m,U}}\\
&\quad+\frac{1}{A_{m,U}}\sum_{u<U}
\mathbb E\!\left[
a_m^u(\varepsilon_m^u)^2
\right]\\
&\quad+\frac{L_E\sigma_E^2}{A_{m,U}}
\sum_{u<U}\mathbb E[(a_m^u)^2].
\label{eq:routed_stationarity}
\end{aligned}
\end{equation}
\end{theorem}

\begin{proof}
Let $g^u=g_{\kappa_m,m}(\theta_m^u)$, $\mu^u=g^u+b_m^u$, and $a_u=a_m^u$.
Smoothness and the conditional variance bound give
\begin{equation}
\begin{aligned}
\mathbb E[F_{\kappa_m,m}(\theta_m^{u+1})\mid\mathcal F_u]
&\leq F_{\kappa_m,m}(\theta_m^u)
-a_u\langle g^u,\mu^u\rangle\\
&\quad+\frac{L_Ea_u^2}{2}
(\|\mu^u\|^2+\sigma_E^2).
\end{aligned}
\end{equation}
Using
$-\langle g^u,\mu^u\rangle
=-\frac12\|g^u\|^2-\frac12\|\mu^u\|^2+\frac12\|b_m^u\|^2$
and $a_u\leq1/L_E$ yields one-step descent with residual
$a_u\|b_m^u\|^2/2+L_Ea_u^2\sigma_E^2/2$.
Taking total expectation, telescoping, and applying Proposition~\ref{prop:routing-fidelity} proves Eq.~\eqref{eq:routed_stationarity}.
\end{proof}

\begin{corollary}[Matched-effective-weight convergence comparison]
\label{cor:matched-expert-convergence}
Consider prototype- and client-level expert updates with the same initialization, effective weights $\{a_m^u\}_{u<U}$, smoothness, and noise terms.
Let their nonnegative directional-error bounds be
$\varepsilon_{m,\mathrm{proto}}^u=2B_g\chi_m^u+\delta_{\mathrm{pat}}$ and
$\varepsilon_{m,\mathrm{client}}^u=2B_g\alpha_m^{\mathrm{client},u}+\delta_{\mathrm{pat}}$.
If $\chi_m^u\leq\alpha_m^{\mathrm{client},u}$ at every active step, the prototype-conditioned right-hand side of Eq.~\eqref{eq:routed_stationarity} is no larger, and is strictly smaller whenever the inequality is strict with positive probability at an active step.
\end{corollary}

\begin{proof}
Under the matched quantities, only the directional-error term differs; the claim follows from
$(\varepsilon_{m,\mathrm{proto}}^u)^2\leq(\varepsilon_{m,\mathrm{client}}^u)^2$.
\end{proof}

Equation~\eqref{eq:routed_stationarity} is a fixed-context expert convergence guarantee.
It separates two failure modes of the actual raw update: $\chi_m$ controls directional contamination, while $A_{m,U}$ records the effective optimization budget received by expert $m$.
If an expert receives negligible routed weight, the first term remains large even when its normalized direction is pure.

\subsection{Persistent and Reset-and-Average Router Optimization}
\label{app:router_persistence}

Freeze the experts and write $\varphi$ for the router parameters.
For client $i$, let $\pi_{i,c}=|\mathcal B_{i,c}|/|\mathcal D_i|$ be bucket $c$'s sample fraction and let
\begin{equation}
\begin{aligned}
R_i(\varphi)&=\sum_c\pi_{i,c}R_{i,c}(\varphi),\\
g_i(\varphi)&=\nabla R_i(\varphi),
&g_{i,c}(\varphi)&=\nabla R_{i,c}(\varphi).
\end{aligned}
\label{eq:router_client_objective}
\end{equation}

The persistent router performs SGD on one shared state.
Randomized sample-count-proportional sampling draws $c_s$ according to $\Pr(c_s=c)=\pi_{i,c}$ and is conditionally unbiased.
For the implemented deterministic sample-count-balanced schedule, define its conditional order bias
\begin{equation}
q_s(\varphi_s)
=\mathbb E[\widehat g_{i,c_s}(\varphi_s)\mid\varphi_s]-g_i(\varphi_s).
\label{eq:persistent_order_bias}
\end{equation}
Thus $q_s=0$ for the randomized reference, while deterministic interleaving is covered through an explicit residual.
The reset alternative instead maintains bucket-local states, evaluates subsequent gradients at those dispersed states, and averages them after the local block.
The gradient-dominance condition below is a standard route to linear SGD convergence, while the resulting dispersion term is analogous to client drift in local federated optimization; the main text provides the relevant citations.

\begin{assumption}[Router smoothness, gradient dominance, and noise]
\label{ass:router-optimization}
Each $R_{i,c}$ is $L_R$-smooth, and $R_i$ satisfies the $\mu$-Polyak--\L{}ojasiewicz condition
\begin{equation}
\|\nabla R_i(\varphi)\|^2
\geq2\mu(R_i(\varphi)-R_i^\star).
\label{eq:router_pl}
\end{equation}
Persistent gradients have conditional mean $g_i(\varphi_s)+q_s(\varphi_s)$, and their centered conditional variance is at most $\sigma_R^2$.
For reset updates, the weighted aggregate of bucket-gradient noise has conditional second moment at most $\sigma_R^2$.
\end{assumption}

\begin{theorem}[Persistent shared-state router convergence]
\label{thm:persistent-router-appendix}
For $\varphi_{s+1}=\varphi_s-\eta_R\widehat g_{i,c_s}(\varphi_s)$ with $0<\eta_R\leq1/L_R$,
\begin{equation}
\begin{aligned}
\mathbb E[R_i(\varphi_S)-R_i^\star]
&\leq \mathcal E_{i,S}+\mathcal O_{i,S},\\
\mathcal E_{i,S}
&:=(1-\mu\eta_R)^S(R_i(\varphi_0)-R_i^\star)
+\frac{L_R\eta_R\sigma_R^2}{2\mu},\\
\mathcal O_{i,S}
&:=\frac{\eta_R}{2}\sum_{s=0}^{S-1}
(1-\mu\eta_R)^{S-1-s}
\mathbb E\|q_s(\varphi_s)\|^2.
\end{aligned}
\label{eq:persistent_router_convergence}
\end{equation}
\end{theorem}

\begin{proof}
Let $\mu_s=g_i(\varphi_s)+q_s(\varphi_s)$.
Smoothness, the centered variance bound, and
$-\langle g_i,\mu_s\rangle
=-\frac12\|g_i\|^2-\frac12\|\mu_s\|^2+\frac12\|q_s\|^2$
imply
\begin{equation}
\begin{aligned}
\mathbb E[R_i(\varphi_{s+1})\mid\varphi_s]
&\leq R_i(\varphi_s)
-\frac{\eta_R}{2}\|g_i(\varphi_s)\|^2
+\frac{\eta_R}{2}\|q_s(\varphi_s)\|^2\\
&\quad+\frac{L_R\eta_R^2}{2}\sigma_R^2.
\end{aligned}
\end{equation}
Here the nonpositive coefficient of $\|\mu_s\|^2$ was dropped using $\eta_R\leq1/L_R$.
Applying Eq.~\eqref{eq:router_pl} and unrolling proves Eq.~\eqref{eq:persistent_router_convergence}.
\end{proof}

For the clean matched-weight reset abstraction, set $q_s=0$, initialize every bucket state at the same router $\varphi_{c,0}=\bar\varphi_0$, and take local SGD steps
$\varphi_{c,s+1}=\varphi_{c,s}-\eta_R\widehat g_{i,c}(\varphi_{c,s})$.
This equal-depth construction isolates state fragmentation: $s$ indexes depth along each bucket-local path, so the comparison matches depth rather than total gradient evaluations.
Define the virtual weighted mean, local-model dispersion, and consensus-gradient error
\begin{equation}
\begin{aligned}
\bar\varphi_s&=\sum_c\pi_{i,c}\varphi_{c,s},\\
D_s&=\sum_c\pi_{i,c}\|\varphi_{c,s}-\bar\varphi_s\|^2,\\
e_s&=\sum_c\pi_{i,c}
\bigl(g_{i,c}(\varphi_{c,s})-g_{i,c}(\bar\varphi_s)\bigr).
\end{aligned}
\label{eq:reset_consensus_definitions}
\end{equation}
The mean reset trajectory therefore follows a biased client-objective direction:
\begin{equation}
\mathbb E\!\left[
\sum_c\pi_{i,c}\widehat g_{i,c}(\varphi_{c,s})
\,\middle|\,\{\varphi_{c,s}\}_c
\right]
=g_i(\bar\varphi_s)+e_s.
\label{eq:reset_biased_direction}
\end{equation}

\begin{theorem}[Reset-and-average router convergence with consensus error]
\label{thm:reset-router-appendix}
Under Assumption~\ref{ass:router-optimization} and $0<\eta_R\leq1/L_R$,
\begin{equation}
\begin{aligned}
\mathbb E[R_i(\bar\varphi_S)-R_i^\star]
&\leq
(1-\mu\eta_R)^S(R_i(\bar\varphi_0)-R_i^\star)
+\frac{L_R\eta_R\sigma_R^2}{2\mu}\\
&\quad+
\frac{\eta_R}{2}\sum_{s=0}^{S-1}
(1-\mu\eta_R)^{S-1-s}\mathbb E\|e_s\|^2,
\label{eq:reset_router_convergence}
\end{aligned}
\end{equation}
where
\begin{equation}
\|e_s\|^2\leq L_R^2D_s.
\label{eq:consensus_error_bound}
\end{equation}
Thus reset-and-average can incur an optimization term generated by bucket-local model dispersion; the persistent shared-state bound in Eq.~\eqref{eq:persistent_router_convergence} has no such consensus term.
\end{theorem}

\begin{proof}
The weighted mean update is
$\bar\varphi_{s+1}=\bar\varphi_s-\eta_R\sum_c\pi_{i,c}\widehat g_{i,c}(\varphi_{c,s})$.
Equation~\eqref{eq:reset_biased_direction} and the same biased-descent identity used in Theorem~\ref{thm:persistent-router-appendix} give
\begin{equation}
\begin{aligned}
\mathbb E[R_i(\bar\varphi_{s+1})-R_i^\star]
&\leq(1-\mu\eta_R)
\mathbb E[R_i(\bar\varphi_s)-R_i^\star]\\
&\quad+\frac{\eta_R}{2}\mathbb E\|e_s\|^2
+\frac{L_R\eta_R^2}{2}\sigma_R^2.
\end{aligned}
\end{equation}
Unrolling proves Eq.~\eqref{eq:reset_router_convergence}.
For Eq.~\eqref{eq:consensus_error_bound}, Jensen's inequality and $L_R$-smoothness give
\begin{equation}
\begin{aligned}
\|e_s\|^2
&\leq\sum_c\pi_{i,c}
\|g_{i,c}(\varphi_{c,s})-g_{i,c}(\bar\varphi_s)\|^2\\
&\leq L_R^2\sum_c\pi_{i,c}
\|\varphi_{c,s}-\bar\varphi_s\|^2
=L_R^2D_s.
\end{aligned}
\end{equation}
\end{proof}

For the randomized reference $q_s=0$, matched initialization, step size, depth, and noise, the certificate gap is the nonnegative consensus sum in Eq.~\eqref{eq:reset_router_convergence}, and is strict when that sum is positive.
For the deterministic persistent schedule, the comparison is between $\mathcal O_{i,S}$ and the reset residual; their magnitudes determine the bound ordering.

The dispersion term is directly tied to heterogeneity across bucket router objectives.
At a common initialization $\varphi$, define
\begin{equation}
\zeta_{i,\mathrm{within}}^2(\varphi)
=\sum_c\pi_{i,c}\|g_{i,c}(\varphi)-g_i(\varphi)\|^2.
\label{eq:within_router_heterogeneity}
\end{equation}

\begin{corollary}[First-step reset dispersion]
\label{cor:first-reset-dispersion}
For deterministic full gradients and one reset step from a common initialization,
\begin{equation}
D_1=\eta_R^2\zeta_{i,\mathrm{within}}^2(\varphi).
\label{eq:first_reset_dispersion}
\end{equation}
Hence heterogeneous bucket gradients create the consensus-error source that enters Eq.~\eqref{eq:reset_router_convergence} from subsequent local steps onward.
\end{corollary}

\begin{proof}
After one step,
$\varphi_{c,1}=\varphi-\eta_Rg_{i,c}(\varphi)$ and
$\bar\varphi_1=\varphi-\eta_Rg_i(\varphi)$.
Substitution into $D_1$ proves the result.
\end{proof}

Theorem~\ref{thm:reset-router-appendix} isolates state fragmentation under matched effective pattern weights.
If reset endpoints trained for $H_c$ steps are aggregated with coefficients $a_c$, their first-order normalized direction instead uses
\begin{equation}
\widetilde\pi_c
=\frac{a_cH_c}{\sum_ja_jH_j},
\qquad
b_{\mathrm{wt}}(\varphi)
=\sum_c(\widetilde\pi_c-\pi_{i,c})g_{i,c}(\varphi).
\label{eq:reset_weight_mismatch}
\end{equation}
This bias precedes dispersion; a multi-step error also contains a $\widetilde\pi$-weighted consensus term.
If both $H_c\propto\pi_{i,c}$ and $a_c\propto\pi_{i,c}$, then $\widetilde\pi_c\propto\pi_{i,c}^2$; step-normalized endpoint deltas or any rule enforcing $\widetilde\pi_c=\pi_{i,c}$ restore matched weighting.
More explicitly, if $\|g_{i,c}(\varphi)\|\leq G_R$, then
\begin{equation}
\|b_{\mathrm{wt}}(\varphi)\|
\leq G_R\|\widetilde\pi-\pi_i\|_1.
\label{eq:reset_weight_bias_bound}
\end{equation}
For an equal-depth reset using fixed weights $\widetilde\pi$, define
\begin{equation}
\begin{aligned}
\bar\varphi_s^{\widetilde\pi}
&=\sum_c\widetilde\pi_c\varphi_{c,s},\\
e_s^{\widetilde\pi}
&=\sum_c\widetilde\pi_c
\bigl(g_{i,c}(\varphi_{c,s})
-g_{i,c}(\bar\varphi_s^{\widetilde\pi})\bigr).
\end{aligned}
\end{equation}
Its expected aggregate direction is
$g_i(\bar\varphi_s^{\widetilde\pi})
+b_{\mathrm{wt}}(\bar\varphi_s^{\widetilde\pi})
+e_s^{\widetilde\pi}$.
Thus the biased-descent proof of Theorem~\ref{thm:reset-router-appendix} replaces $e_s$ by
$b_{\mathrm{wt}}+e_s^{\widetilde\pi}$, so weight mismatch enters the same geometrically weighted optimization residual before the gap-calibration transfer.
In the evaluated \emph{w/o client-router} ablation, $H_c$ is allocated approximately by bucket size (with minimum-step and rounding corrections) and endpoint coefficients are again proportional to bucket sample count.
Thus the aggregation-granularity ablation in Table~\ref{tab:ablation_asymmetry} evaluates the implemented reset-and-reweight procedure, whose residual includes both weight mismatch and consensus error.

\subsection{Router Optimization and Sparse-Inference Risk}
\label{app:router_to_sparse}

Let $z(x)\in\{1,\ldots,K\}$ be the latent pattern of $x$, let $m^\dagger(x)=\psi(z(x))$ be its oracle expert, and let $\ell_m(x,y)$ be the loss using only expert $m$.
Define $R_{i,\mathrm{oracle}}=\mathbb E_{(x,y)\sim\mathcal D_i}[\ell_{m^\dagger(x)}(x,y)]$, let $r_m(x;\varphi)$ be the router probability, and define the oracle-probability deficit
\begin{equation}
S_i(\varphi)
=\mathbb E_{x\sim\mathcal D_i}[1-r_{m^\dagger(x)}(x;\varphi)].
\label{eq:oracle_mass_deficit}
\end{equation}

\begin{assumption}[Gap-calibrated task loss]
\label{ass:gap-calibration}
The fixed experts have an identifiable oracle assignment, and the soft-routed task objective is calibrated to that assignment:
\begin{equation}
\begin{gathered}
c_{\mathrm{gap}}S_i(\varphi)
\leq R_i(\varphi)-R_{i,\mathrm{oracle}}+\beta_{\mathrm{mix}},\\
c_{\mathrm{gap}}>0,\qquad\beta_{\mathrm{mix}}\geq0.
\end{gathered}
\label{eq:gap_calibration}
\end{equation}
Define the nonnegative router-class approximation gap
$A_i:=[R_i^\star-R_{i,\mathrm{oracle}}]_+$.
For a linear mixture of fixed expert losses with pointwise non-oracle loss gap at least $\Delta$, Eq.~\eqref{eq:gap_calibration} holds with $c_{\mathrm{gap}}=\Delta$ and $\beta_{\mathrm{mix}}=0$.
\end{assumption}

Let $s_m(x;\varphi)$ denote router logits and
\begin{equation}
\operatorname{mar}_{\varphi}(x)
=s_{m^\dagger(x)}(x;\varphi)-\max_{j\neq m^\dagger(x)}s_j(x;\varphi).
\end{equation}
Define $R_{\tau,i}(\varphi)=\Pr_{x\sim\mathcal D_i}[\operatorname{mar}_{\varphi}(x)\leq\tau]$.

\begin{proposition}[Gap-calibrated margin risk]
\label{prop:gap-margin-risk}
Under Assumption~\ref{ass:gap-calibration}, for every $\tau\in\mathbb R$,
\begin{equation}
R_{\tau,i}(\varphi)
\leq
\frac{1+e^\tau}{c_{\mathrm{gap}}}
\bigl(R_i(\varphi)-R_i^\star+A_i+\beta_{\mathrm{mix}}\bigr).
\label{eq:gap_margin_risk}
\end{equation}
\end{proposition}

\begin{proof}
If $\operatorname{mar}_{\varphi}(x)\leq\tau$, then one competing logit is at least
$s_{m^\dagger(x)}-\tau$, so
$r_{m^\dagger(x)}(x;\varphi)\leq(1+e^{-\tau})^{-1}$ and
$1-r_{m^\dagger(x)}(x;\varphi)\geq(1+e^\tau)^{-1}$.
Therefore
$\mathbf1\{\operatorname{mar}_{\varphi}(x)\leq\tau\}
\leq(1+e^\tau)(1-r_{m^\dagger(x)}(x;\varphi))$.
Taking expectation, applying Eq.~\eqref{eq:gap_calibration}, and using
$R_i(\varphi)-R_{i,\mathrm{oracle}}
\leq R_i(\varphi)-R_i^\star+A_i$
proves the result.
\end{proof}

Let $\mathcal E_{i,S}$ and $\mathcal O_{i,S}$ be defined in Eq.~\eqref{eq:persistent_router_convergence}, and define
$\mathcal C_{i,S}=\frac{\eta_R}{2}\sum_{s=0}^{S-1}(1-\mu\eta_R)^{S-1-s}\mathbb E\|e_s\|^2$.

\begin{corollary}[Persistent and reset oracle-routing risk]
\label{cor:persistent-reset-risk}
Under Assumptions~\ref{ass:router-optimization}--\ref{ass:gap-calibration}, a common initialization, and the matched-weight reset construction of Theorem~\ref{thm:reset-router-appendix}, for every $\tau\in\mathbb R$,
\begin{equation}
\begin{aligned}
\mathbb E R_{\tau,i}(\varphi_S)
&\leq\frac{1+e^\tau}{c_{\mathrm{gap}}}
(\mathcal E_{i,S}+\mathcal O_{i,S}+A_i+\beta_{\mathrm{mix}}),\\
\mathbb E R_{\tau,i}(\bar\varphi_S)
&\leq\frac{1+e^\tau}{c_{\mathrm{gap}}}
(\mathcal E_{i,S}+\mathcal C_{i,S}+A_i+\beta_{\mathrm{mix}}).
\label{eq:persistent_reset_margin}
\end{aligned}
\end{equation}
\end{corollary}

This follows from Proposition~\ref{prop:gap-margin-risk} and Theorems~\ref{thm:persistent-router-appendix}--\ref{thm:reset-router-appendix}.
Hence persistent shared-state SGD has low oracle-routing surrogate risk when $\mathcal E_{i,S}$, $\mathcal O_{i,S}$, $A_i$, and $\beta_{\mathrm{mix}}$ are small.
In the randomized reference $\mathcal O_{i,S}=0$, whereas matched reset additionally propagates $\mathcal C_{i,S}$.
The weight-mismatched extension above replaces this clean consensus residual by one containing
$b_{\mathrm{wt}}+e_s^{\widetilde\pi}$.

Finally, let $\widehat m(x)=\arg\max_ms_m(x;\varphi)$ and $R_{\mathrm{sparse},i}=\mathbb E_{(x,y)\sim\mathcal D_i}[\ell_{\widehat m(x)}(x,y)]$, and assume
\begin{equation}
[\ell_{\widehat m(x)}(x,y)-\ell_{m^\dagger(x)}(x,y)]_+
\leq\Delta_{\max}.
\label{eq:bounded_sparse_excess}
\end{equation}
Since the mismatch event is contained in $\{\operatorname{mar}_{\varphi}(x)\leq0\}$,
\begin{corollary}[Optimization-controlled sparse-routing risk]
\label{cor:optimization-sparse-risk}
\begin{equation}
\begin{aligned}
R_{\mathrm{sparse},i}-R_{i,\mathrm{oracle}}
&\leq\Delta_{\max}R_{0,i}(\varphi)\\
&\leq\frac{2\Delta_{\max}}{c_{\mathrm{gap}}}
\bigl(R_i(\varphi)-R_i^\star+A_i+\beta_{\mathrm{mix}}\bigr).
\label{eq:optimization_sparse_risk}
\end{aligned}
\end{equation}
\end{corollary}

For bounded metrics one may take a global $\Delta_{\max}$; for unbounded token-level cross-entropy, Eq.~\eqref{eq:bounded_sparse_excess} applies to clipped loss or requires a finite conditional excess-loss moment.
The routing heatmap and same-checkpoint sparse-inference ablation empirically probe the reliable-routing regime captured by the bound.

\end{document}